\documentclass[letterpaper]{article} 
\usepackage{aaai2026}  
\usepackage{times}  
\usepackage{helvet}  
\usepackage{courier}  
\usepackage[hyphens]{url}  
\usepackage{graphicx} 
\usepackage{natbib}  
\usepackage{caption} 
\usepackage{algorithm}
\usepackage{algorithmic}

\usepackage{newfloat}
\usepackage{listings}
\DeclareCaptionStyle{ruled}{labelfont=normalfont,labelsep=colon,strut=off} 
\floatstyle{ruled}
\newfloat{listing}{tb}{lst}{}
\floatname{listing}{Listing}
\usepackage[utf8]{inputenc} 
\usepackage[T1]{fontenc}    

\usepackage{microtype}
\usepackage{graphicx}
\usepackage{subfigure}
\usepackage{booktabs}       
\usepackage{url}            
\usepackage{amsfonts}       
\usepackage{nicefrac}       
\usepackage{xcolor}         
\usepackage{multirow}
\usepackage{bm}

\usepackage{amsmath}
\usepackage{amssymb}
\usepackage{mathtools}
\usepackage{amsthm}
\usepackage[capitalize,noabbrev]{cleveref}
\usepackage{enumitem}

\theoremstyle{plain}
\newtheorem{theorem}{Theorem}[section]

\newtheorem{lemma}[theorem]{Lemma}

\theoremstyle{definition}
\newtheorem{definition}[theorem]{Definition}

\theoremstyle{remark}
\newtheorem{remark}[theorem]{Remark}

\newcommand{\R}{\mathbb{R}} 
\newcommand{\M}{\mathcal{M}}
\newcommand{\N}{\mathcal{N}}

\title{Multi-Scale Geometric Autoencoder}
\author{
    Qipeng Zhan,
    Zhuoping Zhou,
    Zexuan Wang,
    Li Shen
}
\affiliations{
University of Pennsylvania
}

\begin{document}

\maketitle

\begin{abstract}
Autoencoders have emerged as powerful models for visualization and dimensionality reduction based on the fundamental assumption that high-dimensional data is generated from a low-dimensional manifold. A critical challenge in autoencoder design is to preserve the geometric structure of data in the latent space, with existing approaches typically focusing on either global or local geometric properties separately. Global approaches often encounter errors in distance approximation that accumulate, while local methods frequently converge to suboptimal solutions that distort large-scale relationships. We propose Multi-Scale Geometric Autoencoder (MAE), which introduces an asymmetric architecture that simultaneously preserves both scales of the geometric structure by applying global distance constraints to the encoder and local geometric constraints to the decoder. Through theoretical analysis, we establish that this asymmetric design aligns naturally with the distinct roles of the encoder and decoder components. Our comprehensive experiments on both synthetic manifolds and real-world datasets demonstrate that MAE consistently outperforms existing methods across various evaluation metrics.
\end{abstract}


\section{Introduction}

The analysis of high-dimensional data represents a significant challenge within the field of machine learning: the extraction of low-dimensional representations that accurately preserve the fundamental structure of complex datasets~\cite{wang2014generalized}. Although autoencoders have been recognized as effective instruments for dimensionality reduction, an important limitation remains—standard autoencoders that are solely optimized for reconstruction error frequently do not capture the inherent geometric structure of the data manifold. This shortcoming can result in representations that distort meaningful relationships present in the data, thereby jeopardizing downstream tasks that depend on these geometric properties~\cite{trofimov2023learning,spae,gropp2020isometric}.


The challenge of preserving geometric structure in learned representations has led to two distinct methodologies in the literature. The first, exemplified by the Structure-Preserving Autoencoder~\cite{spae} and the Witness Autoencoder~\cite{wae}, emphasizes preserving global structure through similarity graphs that encapsulate relationships across the entire dataset. Although effective in maintaining large-scale structures, these approaches face a key limitation: heightened sensitivity to graph construction quality due to error accumulation when approximating geodesic distances via shortest paths, especially for distant points linked through multiple edges.


The second approach, exemplified by the Graph Geometry-Preserving Autoencoder (GGAE)~\cite{ggae} and the Geometric Autoencoder (GeomAE)~\cite{geomae}, underscores the importance of preserving local geometric properties within small neighborhoods. These methodologies demonstrate exceptional proficiency in maintaining local structure; however, they frequently converge toward suboptimal solutions that compromise global relationships, as their optimization is primarily concerned with local properties.


Our principal insight reveals that these methodologies are not merely complementary; rather, a fundamental asymmetry exists in the manner in which global and local geometric constraints align naturally with various components of the autoencoder architecture. More specifically, we illustrate that the preservation of global structure is inherently more appropriate for the encoder, whereas local geometric constraints prove to be more effective when applied to the decoder. This observation prompts us to propose the \emph{Multi-scale Geometric Autoencoder (MAE)}, which capitalizes on this asymmetry through the strategic application of constraints: global distance preservation in the encoder and local geometric preservation in the decoder.


Through extensive theoretical analysis and empirical validation, we demonstrate that this asymmetric design successfully addresses the limitations of both purely global and purely local approaches. Our key contributions include:


\begin{enumerate}[leftmargin=*,nosep]
    \item A theoretical framework that establishes the natural alignment between global constraints and encoders versus local constraints and decoders. 
    \item The Multi-scale Geometric Autoencoder architecture optimally combines global distance preservation in the encoder with local isometric/conformal constraints in the decoder, demonstrating superior performance across diverse datasets.
    \item Extensive experimental validation demonstrating significant improvements over existing methods across multiple benchmarks and real-world applications.
\end{enumerate}

Our empirical results demonstrate two key findings. First, the asymmetric application of geometric constraints consistently outperforms traditional uniform constraint approaches across all evaluation metrics. Through comprehensive ablation studies, we establish that both global and local components play essential, complementary roles - removing either component leads to significant performance degradation, highlighting the necessity of our unified framework. Second, we uncover a previously unrecognized relationship between the effectiveness of regularization and latent space dimensionality. This relationship reveals that the optimal balance between global and local constraints varies systematically with the intrinsic dimensionality of the data manifold, providing new insights into the geometric principles underlying dimensional reduction.

The remainder of this paper is organized as follows: Section~\ref{sec:related} reviews related work in geometric deep learning and autoencoder architectures. Section~\ref{sec:prelim} presents the theoretical foundations of our approach. Section~\ref{sec:MAE} introduces our Multi-scale Geometric Autoencoder architecture and training framework. Section~\ref{sec:exp} presents experimental results and analysis.
\section{Related Work} \label{sec:related}




Research on preserving geometric structures in representation learning encompasses various interconnected fields, ranging from classical manifold learning to contemporary deep learning techniques~\cite{wu2022generalized,lin2008riemannian,celledoni2021structure,liu2021deep}. We organize our discussion around two central themes: the foundations of manifold learning and geometric autoencoders.

\paragraph{Manifold Learning and Geometric Structure Preservation}
Classical manifold learning methods established the fundamental principles for preserving geometric structure during dimensionality reduction. ISOMAP~\cite{balasubramanian2002isomap} pioneered global structure preservation through geodesic distance maintenance, while LLE~\cite{roweis2000nonlinear} focused on preserving local geometric relationships. These approaches demonstrated the importance of geometric preservation but were limited by their inability to handle complex, high-dimensional data and their lack of generalization to out-of-sample points. Recent work has extended these ideas to deep learning frameworks, enabling more flexible and scalable solutions.

\paragraph{Geometric Autoencoders}
Modern approaches to geometric structure preservation in autoencoders can be categorized based on their preservation targets:

\textit{Global Structure Preservation:} Methods like SPAE~\cite{spae} and WAE~\cite{wae} focus on maintaining dataset-wide geometric relationships by preserving similarity graphs or distance matrices. While these approaches effectively capture large-scale structures, they face fundamental limitations in graph construction quality and distance approximation accuracy. GRAE~\cite{grae} attempted to address these issues through adaptive graph refinement, but challenges remain in balancing computational complexity with preservation accuracy.

\textit{Local Geometry Preservation:} Approaches such as LOCA~\cite{peterfreund2020local}, GGAE~\cite{ggae}, and GeomAE~\cite{geomae} prioritize preserving differential geometric properties within local neighborhoods. These methods typically impose constraints on the Jacobian matrix to maintain local isometries or conformal mappings. While effective at preserving local structure, they often struggle with global consistency, particularly when the manifold contains complex topological features.
\section{Preliminaries} \label{sec:prelim}
In this section, we present the fundamental mathematical concepts from differential geometry that support our approach. We start with basic definitions and progressively develop the specific geometric properties we intend to preserve in our autoencoder framework.

\subsection{Differentiable Manifold and Local Coordinates}
In differential geometry, an $m$-dimensional manifold $\M$ is called a differentiable manifold if it is locally diffeomorphic to $\R^m$. A chart on $\M$ is an ordered pair $(U,\varphi)$ where $U$ is an open subset of $\M$ and $\varphi:U\rightarrow\varphi(U)\subseteq \R^n$ is a homeomorphism. It provides a local coordinate representation for points in $U$. An atlas is a family of charts $\{(U_\gamma,\varphi_\gamma)\}_\gamma\in\Gamma$ such that $\M = \bigcup_{\gamma\in\Gamma}U_\gamma$. An atlas provides a local coordinate system of $\mathcal{M}$, enabling the use of properties and tools of Euclidean space to study the manifold locally.

\subsection{Riemannian Metric and Isometric mapping}
\label{localiso}
A Riemannian manifold $(\M,g)$ is a differentiable manifold $\M$ equipped with a Riemannian metric $g$, which is a covariant 2-tensor field that gives inner products for each tangent space of the manifold. Let $f: \M\rightarrow \R^n$ ($n\geq m$) be an embedding. This embedding induces the Riemannian metric:
\begin{align}
    g^f=J_f(z)^\top J_f(z),
    \label{induced metric}
\end{align}
where $z$ is a local coordinate system of $\M$ and $J_f(z)$ is the Jacobian matrix of $f$ with respect to $z$. 

A local diffeomorphism $F$ between two Riemannian manifolds $(\M,g^{\M})$ and $(\N,g^{\N})$ is said to be local isometric if and only if $F^* g^\N = g^\M$,
where $F^* g^\N$ denotes the \textbf{pullback metric}. Formally, for a point $p\in \M$, and any tangent vectors $u,v\in T_p\M$, the pullback metric $F^* g^\N$ is defined as:
\begin{align}
    (F^* g^\N)_p(u,v) := g^\M_{F(p)}(d_pF(u),d_pF(v)).
    \label{pullback metric}
\end{align}
$d_pF:T_p\M\rightarrow T_{F(p)}F(\M)$ is the differential of $F$ at $p$. 

Let $\M$ be an $m$-dimensional differentiable manifold with local coordinates $z\in\R^m$ induced by an atlas. Now suppose there is a (local) isometric mapping $F: f(\M)\rightarrow \R^l$ ($l\geq m$), thus $F^*g^{F\circ f}=g^{f}$. By \eqref{pullback metric} and \eqref{induced metric}, this gives:
\begin{align}
    ( J_F(f(z))J_f(z))^\top(J_F(f(z))J_f(z))
    =J_f(z)^\top J_f(z).
    \label{isometric}
\end{align}
Here, the differential of $F$ is exactly the Jacobian matrix of $F$ with respect to $f$. Note that \eqref{isometric} is equivalent to:
\begin{align}
    J_f(z)^\top\Big[J_F(f(z))^\top J_F(f(z))-I_n\Big]J_f(z)
    =\mathbf{0}.
    \label{condition}
\end{align}
Since $\text{span}\{J_f(z_1),\cdots,J_f(z_l)\}=T_{f(z)}f(\M)$, Eq.\eqref{condition} implies that
\begin{align}
    J_F(f(z))^\top J_F(f(z))\Big|_{T_{f(z)}f(\M)} = I_n\Big|_{T_{f(z)}f(\M)},
    \label{localcondition}
\end{align}
i.e., $J_F(f(z))^\top J_F(f(z))$ behaves like an identity transformation restricted on the tangent space.

\subsection{Distance on Riemannian Manifold}
If a Riemannian manifold $(\M,g)$ is connected, the Riemannian metric $g$ naturally introduces a distance function $d_g:\M\times\M\rightarrow \R$. Such a distance turns $\M$ into a metric space $(\M,d_g)$ whose topology is the same as the original manifold topology (see Appendix \ref{distance} for more details). This point is crucial because it allows us to study both the local (Jacobian matrix) and global (distance) properties of a manifold in a consistent way.

Now, we revisit the concept of isometry in the context of metric spaces. For two Riemannian metric-induced metric space
$(\M,d_{g^\M})$ and $(\N,d_{g^\N})$, a mapping $F:\M\rightarrow\N$ is called isometric if and only if $\forall x,y\in\M$,
\begin{align}
    d_{g^\N}(F(x),F(y))=d_{g^\M}(x,y).
    \label{globalcondition}
\end{align}

\section{Multi-Scale Geometric Autoencoder} \label{sec:MAE}
\subsection{Problem Setting}

The conventional assumption for the dimension reduction task is that the given dataset $X=\{x_1,\cdots,x_N\}$ is sampled from an $m$-dimensional manifold $\M$ embedded in $\R^n$. The goal is to re-embed the data manifold into $\R^n$ such that $m<l\ll n$. We denote the embeddings of $\M$ in $\R^n$ and $\R^l$ as $\M_n$ (ambient representation) and $\M_l$ (latent representation) respectively. An autoencoder is a pair of mappings $(E,D):\R^n\stackrel{E}{\longrightarrow}\R^l\stackrel{D}{\longrightarrow}\R^n$, where $E$ is called the encoder and $D$ is called the decoder. The encoder will re-embed the data from higher dimensional ambient space into a lower dimensional latent space, and the decoder will reconstruct the data based on the latent representation. Usually, both encoder and decoder are parametrized by deep neural networks. 

\subsection{Global and Local Geometric Regularization}
A vanilla autoencoder minimizes only the reconstruction loss:
\begin{align}
    \mathcal{L}_{\text{recon}} = \frac{1}{N} \sum_{i=1}^{N}\|x_i-D(E(x_i))\|_2^2.
    \label{reconloss}
\end{align}
This formulation imposes no constraints on the latent representation, offering no assurance that the encoder will maintain the geometric structure of the data. To achieve a more faithful latent representation, we introduce geometric regularization to encourage both the encoder $E$ and decoder $D$ to be isometric.

\subsubsection{Global Isometric Regularization}
Following \eqref{globalcondition}, computing the exact Riemannian (geodesic) distance requires taking an infimum over infinitely many integrals (see Appendix \ref{lengthofcurve}), which is computationally intractable. In practice, the latent distance $d_{g^{\M_l}}$ is typically approximated by the Euclidean distance for simplicity and interpretability. While this approximation may introduce some deformation in the latent representation, we adopt it in our work as the additional local isometric regularization helps mitigate such deformation.

For ambient spaces, where data manifolds are often curved (e.g., the \emph{Swiss Roll}), we cannot make similar Euclidean assumptions. Instead, we approximate geodesic distances by constructing a k-Nearest-Neighbors graph and computing shortest paths using either the Floyd-Warshall algorithm \cite{floyd} or Dijkstra's algorithm \cite{dijkstra}. This approximation is impractical for the reconstructed representation $D(E(\M_n)$ as it evolves during training, making per-step distance computations prohibitively expensive. Therefore, we apply global isometric regularization only to the encoder $E$.

The global isometric loss $\mathcal{L}_{\text{global}}$ can be evaluated in either absolute terms:
\begin{align}
    \mathcal{L}_{\text{global}}^{\text{abs}} = \frac{2}{N(N-1)}\sum_{i<j}^N (d^\M_{ij}-d^E_{ij})^2,
\end{align}
or relative terms:
\begin{align}
    \label{relative error}
    \mathcal{L}_{\text{global}}^{\text{rel}} = \frac{2}{N(N-1)}\sum_{i<j}^N (\frac{d^\M_{ij}-d^E_{ij}}{d^\M_{ij}})^2,
\end{align}
where $d_{ij}^\M$ denotes the approximated geodesic distance between $x_i$ and $x_j$, and $d_{ij}^E$ denotes the Euclidean distance between $E(x_i)$ and $E(x_j)$.

\subsubsection{Local Isometric Regularization}
Following \eqref{localcondition}, verifying the isometry condition requires access to $T_{f(z)}f(\M)$ or $J_f(z)$, which is impossible due to the unknown nature of the embedding $f$. Therefore, we tighten the condition to:
\begin{align}
    J_F(f(z))^\top J_F(f(z)) = I.
    \label{tight}
\end{align}
However, for the encoder case (i.e., $F=E$), $J_E$ is an $l\times n$ matrix, which means $J_E^\top J_E$ is an $n\times n$ matrix with $\text{rank}(J_E^\top J_E)\leq \min(l,n)=l<n$. This makes it impossible for $J_E^\top J_E$ to be an identity matrix $I_n$. Hence, we apply this local isometric regularization only to the decoder $D$:
\begin{align}
    \mathcal{L}_{\text{local}}=\sum_{i=1}^N\|J_D(E(x_i))^\top J_D(E(x_i)) - I_l\|_F^2.
    \label{localreg}
\end{align}

The total loss for our multi-scale geometric autoencoder combines the reconstruction loss, global isometric loss on $E$, and local isometric loss on $D$:
\begin{align}
    \mathcal{L}_{\text{total}} = \mathcal{L}_{\text{recon}} + \lambda_{\text{global}}\cdot \mathcal{L}_{\text{global}} + \lambda_{\text{local}}\cdot\mathcal{L}_{\text{local}},
\end{align}
where $\lambda_{\text{global}}$ and $\lambda_{\text{local}}$ are weight coefficients.

\subsubsection{Conformal Relaxation}
The local isometric constraint can be overly restrictive in certain cases. For example, a half sphere cannot be isometrically embedded into $\R^2$ (see Appendix \ref{hemisphere}), meaning it cannot be flattened without distortion. To address this limitation, we introduce conformal mappings as a natural relaxation, which preserves angles between curves while allowing for local uniform scaling.

Formally, A local diffeomorphism $F$ between two Riemannian manifolds $(\M,g^{\M})$ and $(\N,g^{\N})$ is locally conformal if $\forall p\in\M$, $(F^* g^\N)_p = \lambda_p g_p^\M$. Apparently, when $\lambda_p \equiv 1$, $F$ reduces to an isometry. Analogous to \eqref{tight}, the tightened local conformal constraint is:
\begin{align}
    H^F(z):=J_F(f(z))^\top J_F(f(z)) = \lambda_z I.
    \label{tightconformal}
\end{align}

This relaxation maintains the local geometric structure of the data, preserving shapes up to scaling, which makes it particularly useful in applications where the relative arrangement of points is more important than absolute distances. Note that $\lambda_z$ is unknown and varies for different $z$, so we can not simply regularize it using matrix norm. A practical approach is to separately constrain the diagonal and off-diagonal elements. We decompose the regularization into two components: 
\begin{itemize}[leftmargin=*,nosep]
    \item one that forces the off-diagonal elements to be close to zero
    \begin{align}
        \mathcal{L}_{\text{off}}=\sum_{i=1}^N\sum_{j\neq k}(H^D_{jk}(E(x_i)))^2,
    \end{align}
    \item and another that promotes uniformity among the diagonal elements without specifying their actual value
    \begin{align}
        \mathcal{L}_{\text{diag}}=\sum_{i=1}^N\sum_{j\neq k}^l (H(E(x_i))_{jj}-H(E(x_i))_{kk})^2.
    \end{align}
\end{itemize}    
Then, the local conformal loss is 
\begin{align}
    \mathcal{L}_{\text{local}}^{\text{con}}=\mathcal{L}_{\text{off}} + \lambda_{\text{diag}}\cdot\mathcal{L}_{\text{diag}},
\end{align}
where $\lambda_{\text{diag}}$ is a weight coefficient.
\section{Experiments} \label{sec:exp}
\subsection{Experimental Setup}
\textbf{Baselines:} We select the two widely used non-parametric methods, UMAP\cite{UMAP} and t-SNE\cite{tsne}, as well as several autoencoder-based models including vanilla AE, SPAE\cite{spae}, GGAE\cite{ggae}, GeomAE\cite{geomae}, GRAE\cite{grae} to be our baseline models. To ensure a fair comparison, we performed a systematic grid search to identify the optimal hyperparameters for each baseline model. For the exact hyperparameters used for each experiment, see Table~\ref{tab:parameter} in the Appendix for detailed hyperparameter settings.

\textbf{Ours:} We use the relative error in \eqref{relative error} as our global loss through all datasets and apply both isometric local loss and conformal local loss to our \emph{Multi-Scale Geometric Autoencoder}, which refers to MAE-iso and MAE-con, respectively. The exact hyperparameters used are presented in Table~\ref{tab:parameter_MAE} in the Appendix.

\textbf{Datasets:} To thoroughly evaluate our approach, we conduct experiments on two synthetic datasets (\emph{Swiss Roll}, \emph{Toroidal Helix}) and three real-image datasets (\emph{dSprites}, \emph{Teapot}, and \emph{Objective Tracking}).

\textbf{Warm Up and Weight Decay:} In our experiments, we observed that using a larger $\lambda_{\text{global}}$ at the beginning of training helps the model escape local optima more quickly. However, as training progresses, the global loss begins to dominate the local loss, preventing the model from reaching its best possible performance if  $\lambda_{\text{global}}$ remains large. To address this, we employ warm-up and weight-decay strategies. Specifically, we initialize  $\lambda_{\text{global}}$ with a large value and omit the local loss during the first $T=120$ epochs (warm-up). Meanwhile, we multiply $\lambda_{\text{global}}$ by a decay factor $exp(-\alpha\cdot n_{\text{epoch}})$. These techniques allow the model to converge faster and achieve better performance.


\textbf{Evaluation:}
We qualitatively evaluate our proposed model and other baselines by visualizing the latent data embedded in each post-trained model and comparing them with ideal cases.
In addition, we assess the accuracy of the model by calculating the mean square error between the reconstructed data and the original data. To evaluate how well the learned latent representations preserve the geometric structure of data, we use metrics from \cite{art} and \cite{tae}. The metrics are 1. \emph{kNN}, which calculates the \emph{kNN} recall from latent space to original space, and 2. $KL_\sigma(\sigma\in\{0.01,0.1,1\})$, which measures the Kullback-Leibler divergence based on density estimates in latent and original spaces with different length-scales. In their original form, these metrics are intended to gauge how accurately Euclidean distances in the data space are preserved in the latent representation. In our work, we adapt them by replacing these Euclidean distances with the approximated geodesic distances derived from the similarity graph. Consequently, \emph{kNN} and $KL_{0.01}$ primarily capture local geometry, whereas  $KL_{1}$ measure more global structures. $KL_{0.1}$, serving as an intermediate metric, balances both local and global geometric fidelity. See Appendix~\ref{metrics}  for more details about the metrics.

\subsection{Results}
\begin{table*}[ht]
\centering
\caption{Quantitative evaluation on the Swiss Roll dataset comparing our MAE approach with baseline methods. Lower values ($\downarrow$) are better for reconstruction error (recon.) and KL-divergence, while higher values ($\uparrow$) are better for k-nearest neighbor preservation (KNN). MAE achieves superior performance across metrics measuring geometric preservation. The best performance for each metric is shown in bold, and the second-best is underlined.}
\vspace{-0pt}
\begin{tabular}{cccccc}
\toprule
Method & recon. ($\downarrow$) & kNN ($\uparrow$) & KL$_{0.01}$ ($\downarrow$) & KL$_{0.1}$ ($\downarrow$) & KL$_{1}$ ($\downarrow$) \\
\midrule
MAE-iso & \underline{2.80e-4} & \textbf{9.73e-1} & \textbf{1.02e-2} & 3.91e-2 & 2.80e-3\\
MAE-con & \textbf{2.67e-4} & \underline{9.69e-1} & \underline{1.71e-2} & 4.34e-2 & 2.88e-3 \\
\midrule
GeomAE & 3.25e-1 & 5.45e-1 & 2.79e-2 & \underline{6.20e-3} & \textbf{1.15e-3}\\
GGAE & 2.65e-2 & 6.81e-1 & 1.43e-1 & 5.07e-2 & 2.43e-3\\
GRAE & 2.14e-2 & 6.78e-1 & 4.47e-1 & 8.15e-2 & 2.34e-3\\
SPAE & 6.05e-4 & 9.44e-1 & 1.71e-2 & 4.43e-2 & 2.94e-3\\
TSNE & NaN & 8.31e-1 & 1.81e-2 & \textbf{3.18e-2} & \underline{1.75e-3}\\
UMAP & NaN & 8.02e-1 & 1.43e-2 & 4.27e-2 & 3.43e-3\\
Vanilla AE & 2.89e-2 & 6.32e-1 & 5.18e-1 & 8.96e-2 & 2.61e-3\\
\bottomrule
\end{tabular}
\label{tab:numerical_swissroll}
\end{table*}
\subsubsection{\emph{Swiss Roll}}
\begin{figure*}[ht]
    \centering
    \includegraphics[width=0.7\linewidth]{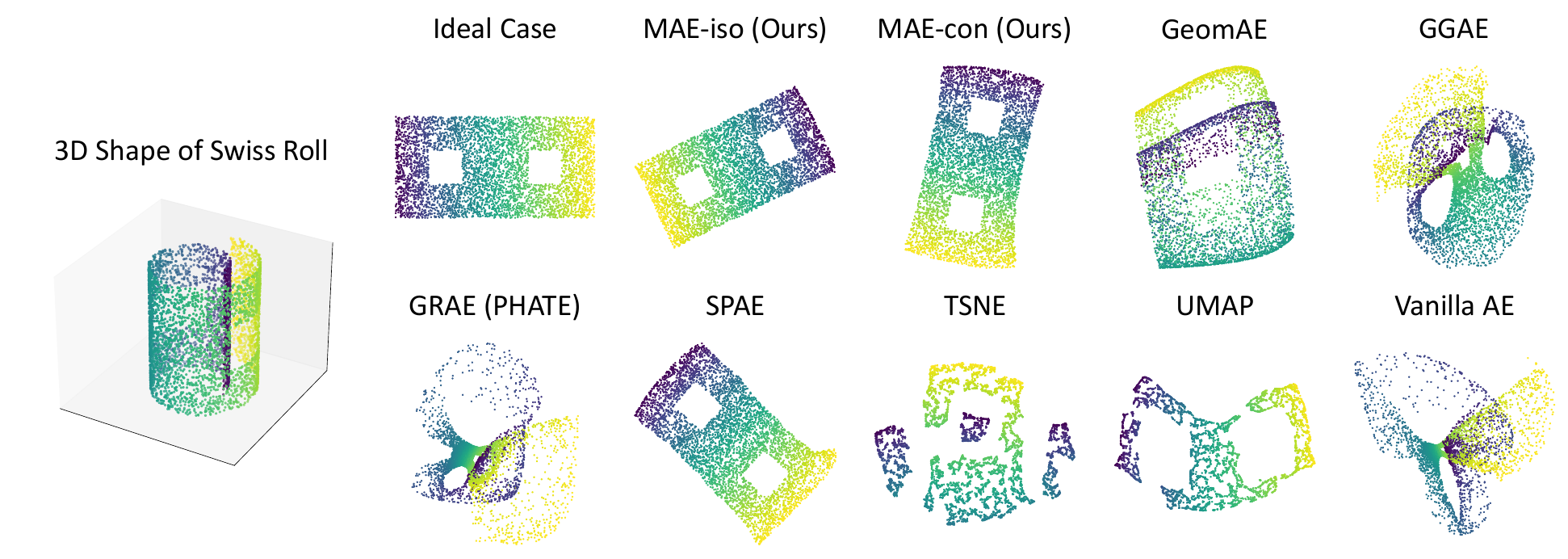}
    \vspace{-0pt}
    \caption{Left: The Swiss Roll manifold visualized in original 3D space and its ideal 2D unwrapped form. Other panels: Learned 2D latent representations from our MAE variants (MAE-iso and MAE-con) and baseline methods. Both MAE variants closely match the ideal unwrapped structure, with SPAE showing partial success, while other methods fail to preserve the manifold's topology.}
    \label{fig:swissroll}
\end{figure*}
The Swiss Roll dataset is a synthetic manifold constructed by sampling points in three-dimensional space, with each point lying on a two-dimensional, spiral-shaped surface, as shown in Figure \ref{fig:swissroll}. Conceptually, this surface can be viewed as a ``rolled sheet'', where the distance between two points depends on paths confined to the surface. One often constructs a k-nearest neighbor graph based on local Euclidean distances between nearby points to approximate these geodesic distances. In our experiment, the spiral contains two empty regions, which forces paths on the k-NN graph to curve around it, creating a mismatch between local connectivity and the global geodesics of the manifold.

\textbf{Result:} Figure \ref{fig:swissroll} shows the two-dimensional latent spaces learned by our method and several benchmarks, while Table \ref{tab:numerical_swissroll} reports the corresponding numerical results. Among all methods evaluated, only the two MAE variants and SPAE recover latent structures that closely resemble the ideal Swiss Roll geometry. MAE-iso aligns most closely with the ground truth, followed by MAE-con. SPAE displays minor boundary distortions but otherwise preserves the manifold’s fundamental shape. Numerical evaluation is shown in Table \ref{tab:numerical_swissroll}.

\subsubsection{\emph{Toroidal Helix}}
\begin{figure*}[ht]
    \centering
    \includegraphics[width=0.7\linewidth]{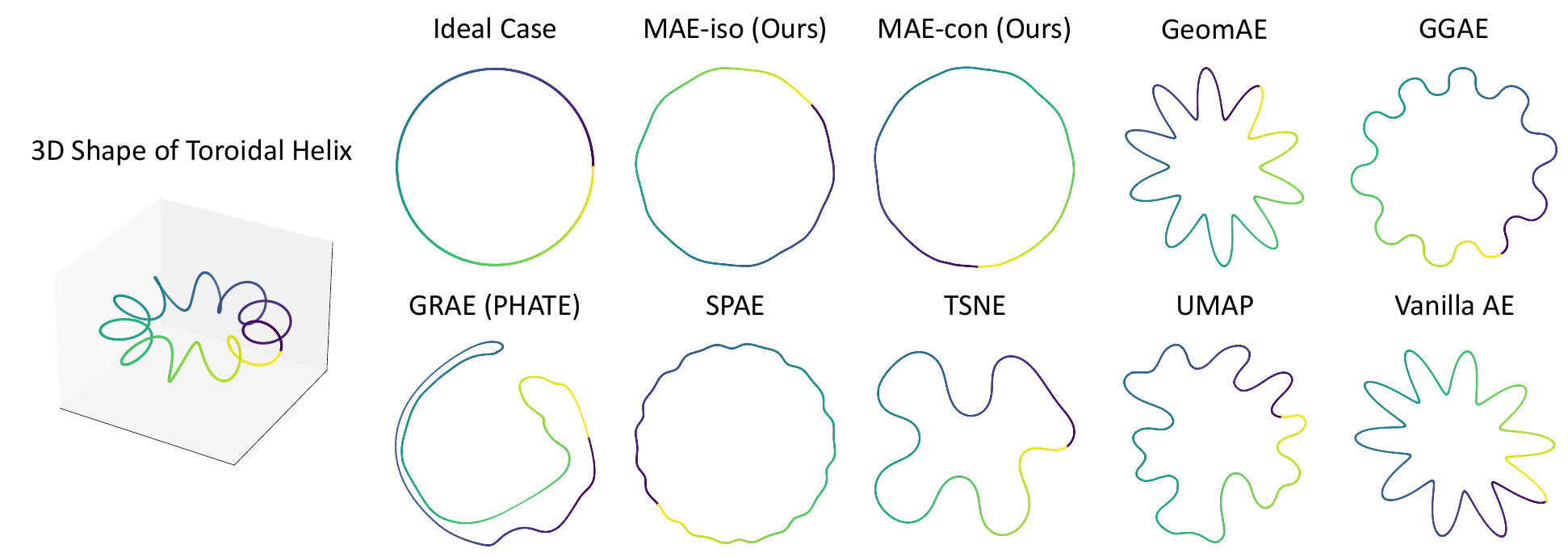}
    \vspace{-0pt}
    \caption{Top-left: The Toroidal Helix manifold shown in its original 3D space and its ideal 2D unwrapped representation. Other panels: Learned 2D latent representations from our MAE variants (MAE-iso and MAE-con) and baseline methods. Both MAE variants successfully recover the global toroidal structure, with SPAE achieving partial preservation. Other methods only preserve local neighborhoods while failing to capture the global topology.}
    \label{fig:helix}
\end{figure*}
The Toroidal Helix dataset is generated by creating points along a helical structure that winds around a torus—a donut-shaped surface. This structure twists in three-dimensional space, yielding complex, spiraled patterns that inherently lie on a curved manifold. Although the data are embedded in $\mathbb{R}^3$, the effective latent representation can be regarded as (approximately) a circle in two dimensions.

\textbf{Result:} Figure \ref{fig:helix} shows the two-dimensional latent spaces learned by our method and several benchmarks, while Table \ref{tab:numerical_helix} presents the corresponding numerical results. Among the methods evaluated, only the two MAE variants and SPAE recover latent structures that closely resemble the ideal helix geometry. MAE-iso and MAE-con align most closely with the ground truth, followed by SPAE, which exhibits minor local distortions compared to the ideal case. In contrast, GeomAE, GGAE, VanillaAE, and UMAP capture the local structure but fail to unfold the helix fully. Numerical evaluation is detailed in Appendix~\ref{app:exp_results}.

\subsubsection{\emph{Teapot}}
\begin{figure*}[ht]
    \centering
    \includegraphics[width=0.7\linewidth]{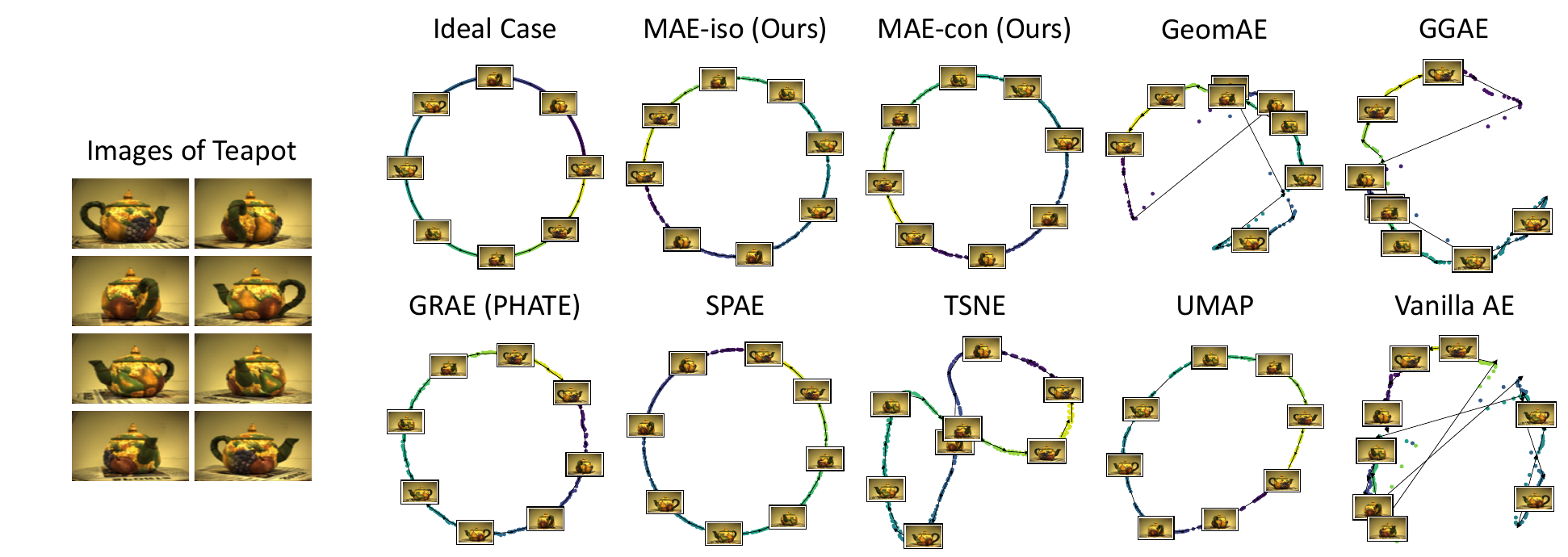}
    \vspace{-0pt}
    \caption{Left: Teapot images rendered from different viewing angles in original space and their ideal 2D manifold structure. Other panels: Learned 2D latent representations from our MAE variants (MAE-iso and MAE-con) and baseline methods. MAE variants, GRAE(PHATE), and SPAE successfully capture the underlying view-dependent manifold structure, while other methods fail to preserve the circular topology.}
    \label{fig:teapot}
\end{figure*}
The Teapot dataset \cite{weinberger2004learning} provides a high-dimensional yet geometrically simple manifold. It consists of 400 color images of a teapot, each with a resolution of 76 $\times$ 101 and three channels of color information per pixel. These images are captured by rotating the teapot through a full 360$^\circ$ in the plane, making the adequate parameter space a one-dimensional circle. Each image can be associated with a single rotation angle, and therefore, the entire dataset could lie on a two-dimensional circle.

\textbf{Result:} Figure \ref{fig:teapot} displays the two-dimensional latent spaces learned by our method and several benchmarks, while Table \ref{tab:numerical_teapot} provides the corresponding numerical results. Among all approaches evaluated, the two MAE variants, GRAE(PHATE), SPAE, and UMAP, successfully recover a two-dimensional loop structure. In contrast, the other methods fail to preserve this topological feature in their latent representations. Numerical evaluation is shown in Appendix~\ref{app:exp_results}.

\subsubsection{\emph{Object Tracking}}
\begin{figure*}[!ht]
    \centering
    \includegraphics[width=0.7\linewidth]{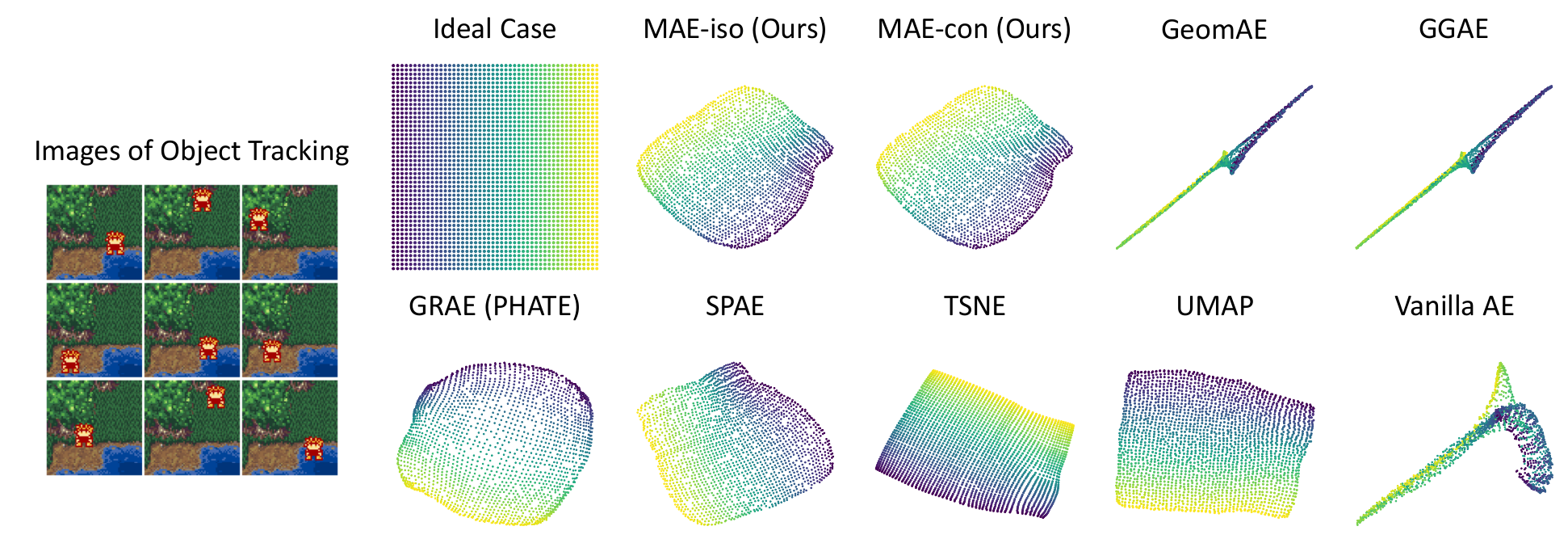}
    \vspace{-0pt}
    \caption{Left: Sample frames from the Object Tracking dataset and their corresponding idealized 2D manifold structure. Remaining panels: Two-dimensional latent space visualizations learned by our Multi-scale Geometric Autoencoder variants (MAE-iso and MAE-con) compared to baseline methods.}
    \label{fig:object}
\end{figure*}

The Object Tracking dataset \cite{grae} consists of approximately 2,000 RGB images, each containing a 16×16 character moving against a 64×64 background with added Gaussian noise. As the character shifts in the horizontal and vertical directions, each image can be associated with an $(x, y)$ position on the background. These two parameters define the intrinsic two-dimensional manifold of the dataset, making it an ideal benchmark for examining how well algorithms can recover and represent simple planar motion from high-dimensional image inputs. 

\textbf{Result:}
Figure \ref{fig:object} displays the two-dimensional latent representation learned by our method and several benchmarks,
while Table \ref{tab:numerical_objective} provides the corresponding numerical results. The ideal latent representation for this experiment is a square. Among all the methods evaluated, most learned a latent space similar to the ideal case, except for GGAE, Geom-AE, and Vanilla-AE. Numerical evaluation is shown in Table \ref{tab:numerical_objective} in Appendix~\ref{app:exp_results}.

\subsubsection{\emph{dSprites}}
\begin{figure*}[ht]
    \centering
    \includegraphics[width=0.9\linewidth]{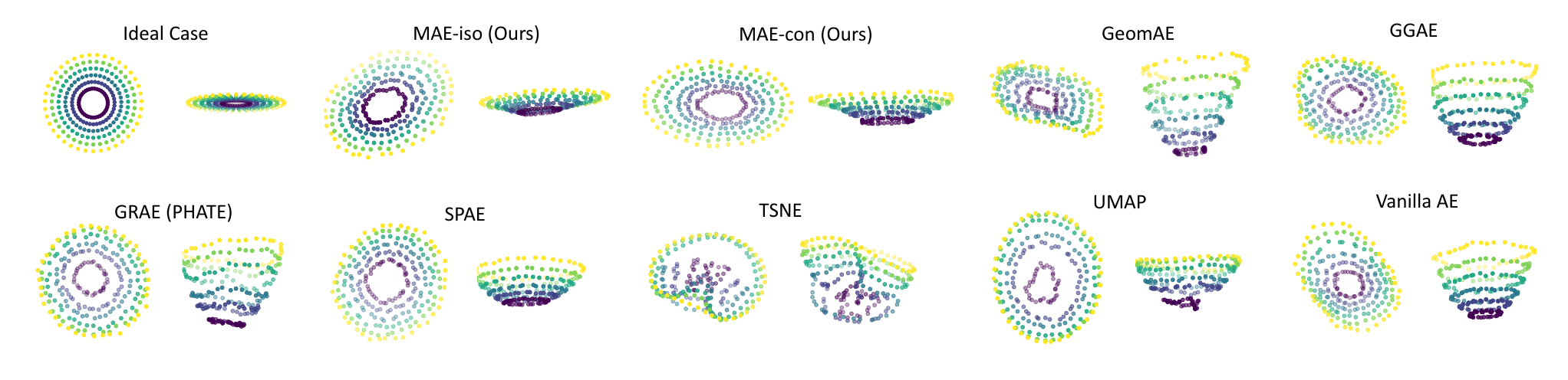}
    \vspace{-0pt}
    \caption{Left: dSprites dataset showing original images and their ideal 2D latent organization based on shape, scale, and rotation attributes. Other panels: Learned 2D latent representations from our MAE variants (MAE-iso and MAE-con) and baseline methods.}
    \label{fig:dsprite}
\end{figure*}

\begin{table*}[th]
\centering
\caption{Quantitative evaluation comparing MAE variants (MAE-iso and MAE-con) against autoencoders with exclusively global or local geometric constraints. Results demonstrate the advantages of combining both constraint types in MAE. The best performance for each metric is shown in bold, and the second-best is underlined.}
\vspace{-0pt}
\begin{tabular}{cccccc}
\toprule
Method & recon. ($\downarrow$) & kNN ($\uparrow$) & KL$_{0.01}$ ($\downarrow$) & KL$_{0.1}$ ($\downarrow$) & KL$_{1}$ ($\downarrow$) \\
\midrule
MAE-iso & \underline{2.80e-4} & \textbf{9.73e-1} & \textbf{1.02e-2} & 3.91e-2 & 2.80e-3\\
MAE-con & \textbf{2.67e-4} & \underline{9.69e-1} & 1.71e-2 & \underline{4.34e-2} & \underline{2.88e-3} \\
Global only & 1.09e-3 & 9.35e-1 & \underline{1.56e-2} & 4.38e-2 & 2.98e-3\\
Local isometric & 2.84e-1 & 5.98e-1 & 3.38e-2 & \textbf{9.57e-3} & \textbf{4.53e-4} \\
Local conformal & 3.19e-2 & 6.78e-1 & 3.13e-1 & 6.58e-2 & 1.91e-3\\
\bottomrule
\end{tabular}
\label{tab:numerical_ablation}
\end{table*}

The dSprites dataset \cite{dsprites17} is a synthetic benchmark for studying disentangled representation learning. It consists of 64×64 binary images of simple geometric shapes—squares, ellipses, and hearts—systematically generated by varying discrete factors: shape type, color (white/black), orientation, scale, and position $(x, y)$. In our setup, we fix color, shape, and position to be \emph{white}, \emph{heart}, and $(16/31, 16/31)$, respectively. Consequently, the dataset is reduced to hearts varying only by scale and orientation. The ideal latent representation is an annulus. Although it is a 2-dimensional manifold, we set (target) latent dimension to be \textbf{3} to evaluate the performance of each model under a misspecified latent dimension.

\textbf{Result:}
Figure \ref{fig:dsprite} displays the three-dimensional latent representation learned by our method and several benchmarks, while Table \ref{tab:numerical_dsprite} provides the corresponding numerical results.  In this experiment, the ideal scenario is a flat annulus embedded in $\R^3$, so a “flatter” latent representation indicates better performance. While all models, except t-SNE and UMAP, preserved the topology of the data, only the two MAE variants learned an almost flat latent space. This finding highlights the superiority of MAE when the latent dimension is misspecified.

\subsection{Ablation Study}

\begin{figure}[th]
    \centering
    \includegraphics[width=0.9\linewidth]{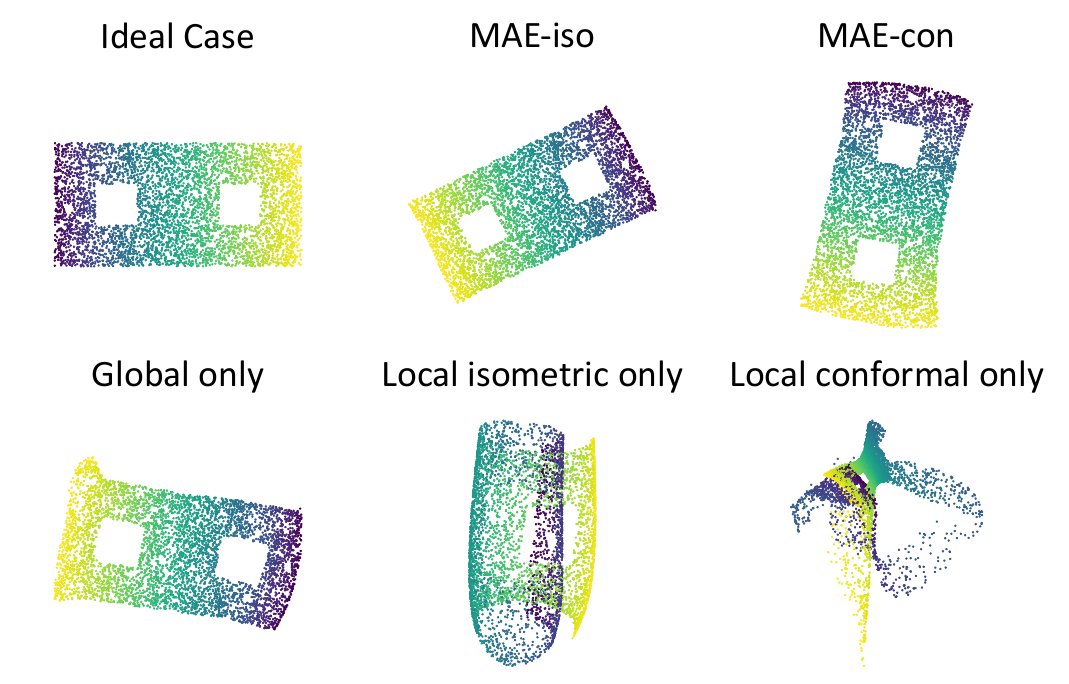}
    \vspace{-0pt}
    \caption{Two-dimensional latent representations comparing our MAE variants (MAE-iso and MAE-con) against autoencoders using solely global or local regularization. MAE-iso achieves near-ideal manifold preservation, while MAE-con exhibits slight geometric distortions. The global-only AE shows more significant distortions, and local-only AEs fail entirely to preserve the manifold structure.}
    \label{fig:ablation}
\end{figure}

We also conducted ablation studies to validate the necessity of both global and local regularization. Back to the Swiss Roll example, here, we run global(local)-only autoencoders by simply changing $\lambda_{\text{local}}~(\text{or}~ \lambda_{\text{global}})$ to 0 and keeping all other settings the same as our multi-scale autoencoder. Figure \ref{fig:ablation} displays the two-dimensional latent representation learned by our method and several benchmarks, while Table \ref{tab:numerical_ablation} provides the corresponding numerical results. This experimental result underscores the importance of incorporating both global and local geometric constraints in our model. Relying solely on local constraints leads the optimization to become trapped in local optima, preventing accurate preservation of the geometric structure in the latent space. Conversely, using only global constraints, while capable of maintaining the overall geometric layout, introduces distortions that result in the loss of local details.

\section{Conclusion} \label{sec:conclu}

This work introduces the Multi-Scale Geometric Autoencoder (MAE), which addresses a fundamental challenge in dimensionality reduction: simultaneously preserving both local and global geometric structures of high-dimensional data manifolds. Our key insight—that global and local geometric constraints naturally align with different components of the autoencoder architecture—led to an asymmetric design that applies global distance preservation in the encoder and local geometric constraints in the decoder.

Through extensive experimentation across synthetic and real-world datasets, we demonstrated that MAE consistently outperforms existing approaches in both quantitative metrics and qualitative assessments. The superior performance was particularly evident in challenging scenarios such as datasets with complex topological features (Swiss Roll, Toroidal Helix) and when the latent dimension is misspecified (dSprites). Our ablation studies revealed that global and local components are essential; removing either leads to significant performance degradation, validating our unified multi-scale approach.

\bibliography{aaai2026}



\newpage
\appendix
\onecolumn
\section{Appendix}
\subsection{Length and Distance on Riemannian Manifolds}
\label{distance}
\subsubsection{Length of Curves}
\label{lengthofcurve}
\begin{definition}[\textbf{Length of Curves}]
    Let $(\M,g)$ be a Riemannian manifold, $\gamma:[a,b]\rightarrow \M$ is a continuous and piecewise smooth curve parametrized by $t$, the length of $\gamma$ is defined as:
    \begin{align}
        L_\gamma = \int_a^b|\gamma'(t)|_gdt
        \label{length}
    \end{align}
\end{definition}
To verify this concept of length is well-defined, we need to show:
\begin{itemize}
    \item \textbf{The integral in \eqref{length} is finite.} The continuity and piecewise smoothness imply that $|\gamma'(t)|_g$ is continuous except on a finite subset of $[a,b]$ and has finite left/right limits at those points. This guarantees the finiteness we need.
    \item \textbf{The length in \eqref{length} is independent of parametrization.} This can be proved by the following lemma:
\end{itemize}
\begin{lemma}
    Suppose $(\M,g)$ is a Riemannian manifold and $\gamma:[a,b]\rightarrow\M$ is a continuous piecewise smooth curve. Let $\varphi:[c,d]\rightarrow[a,b]$ be monotone and differentiable (i.e. $\varphi'>0$ or $\varphi'<0$ on $[c,d]$), then $L_{\gamma\circ\varphi}=L_{\gamma}$.
\end{lemma}
\begin{proof}
    Without loss of generality, we assume that $\varphi'>0$. If $\gamma$ is smooth, then
    \begin{align}
        L_{\gamma\circ\varphi}=L_{\gamma} =& \int_c^d|(\gamma\circ\varphi)'(s)|_gds\\
        =& \int_c^d|\varphi'(s)\gamma'(\varphi(t))|_gds\\
        =& \int_c^d|\gamma'(\varphi(t))|_g\varphi'(s)ds\\
        =& \int_c^d|\gamma'(\varphi(t))|_gdt.
    \end{align}
    If $\gamma$ is only continuous and piecewise smooth, we can apply the same proof on each subinterval on which it is smooth.
\end{proof}
\subsubsection{Riemannian Distance Function}
Given the length of curves on a connected Riemannian manifold, we can naturally define the Riemannian distance between any two points.
\begin{definition}
    Suppose $(\M,g)$ is a connected Riemannian manifold; the Riemannian distance between any two points $x,y\in\M$ is defined as
    \begin{align}
        d_g(x,y) = \inf_\gamma L_r
    \end{align}
    where $\gamma$ is taken over all continuous and piecewise smooth curves with endpoints $x$ and $y$.
\end{definition}
This Riemannian distance turns $\M$ into a metric space $(\M,d_g)$. We will show that $\M$ as a metric space is consistent with its being a manifold. 
\begin{theorem}
    Suppose $(\M,g)$ is a connected Riemannian manifold, $d_g$ is the corresponding Riemannian distance function, then $(\M,d_g)$ is a metric space and the topology of $(\M,d_g)$ is the same as the topology of $(\M,g)$.
    \label{compatible}
\end{theorem}
Before we prove theorem \ref{compatible}, we first introduce two important lemma.
\begin{lemma}
    Suppose $(\M,g)$ is an $n$-dimensional Riemannian manifold. Let $h$ be the Euclidean metric on $\R^n$. Let $(U,\varphi)$ be a chart of $\M$, then for any compact set $K\subset U$, there exist positive constants $c$, $C$ such that for any $p\in K$ and $v\in T_p\M$, we have $c|v|_g\le |\varphi^*v|_h\le C|v|_g$, where $|v|_g=\sqrt{g(v,v)}$ and $|v|_h=\sqrt{h(v,v)}$.
    \label{lemma2}
\end{lemma}
\begin{proof}
    Denote
    \begin{align}
    B_p = \{v\in T_p\M: |v|_g=1\}.
    \end{align}
    Then, $B_p$ is a compact subset of $T_p\M$. So $|\varphi^*v|_h$ is continuous and positive on $B_p$, then there exist positive constants $c$ and $C$ such that 
    \begin{align}
        c\le |\varphi^*v|_h\le C, \forall v\in B_p.
    \end{align}
    Now for any $v\in T_p\M\backslash\{0\}$, $v/|v|_g$, thus
    \begin{align}
       |\varphi^*v|_h=|v|_g\cdot\big|\varphi^*\frac{v}{|v|_g}\big|_h\le C|v|_g,\\
       |\varphi^*v|_h=|v|_g\cdot\big|\varphi^*\frac{v}{|v|_g}\big|_h\ge c|v|_g.
    \end{align}
    for $v=0$, the inequalities are trivial.
\end{proof}

\begin{lemma}
     Suppose $(\M,g)$ is an $n$-dimensional Riemannian manifold. Let $h$ be the Euclidean metric on $\R^n$. Let $(U,\varphi)$ be a chart of $\M$ around $p\in\M$ and $\epsilon>0$ is small enough such that $\overline{B_\epsilon(\varphi(p))}\subset\varphi(U)$ and $\overline{\varphi^{-1}(B_\epsilon(\varphi(p)))}\subset U$. Then, there exist positive constants $c, C$ such that 
     \begin{align}
     cd_g(p,q)\le d_h(\varphi(p),\varphi(q))\le Cd_g(p,q),\quad
         \forall q\in\varphi^{-1}(B_\epsilon(\varphi(p))).
     \end{align}
     \label{lemma3}
\end{lemma}
\begin{proof}
    Let $V=\varphi^{-1}(B_\epsilon(\varphi(p)))$, then $\overline{V}$ is a compact subset of $U$, Lemma \ref{lemma2} implies that there exist constants $c,C$ such that
    \begin{align}
        c|v|_g\le|\varphi^* v|_h\le C|v|_g, \quad \forall q\in\overline{V}\ and\ v\in T_q\M.
    \end{align}
    For any continuous and piecewise smooth curve $\gamma:[a,b]\rightarrow\M$ with $\gamma(a)=p$ and $\gamma(b)=q$, we have
    \begin{align}
        c L_\gamma = \int_a^bc|\gamma'(t)|_gdt\le \int_a^b|(\varphi\circ\gamma)'(t)|_hdt
         =L_{\varphi\circ\gamma}=\int_a^b|(\varphi\circ\gamma)'(t)|_hdt\le\int_a^bC|\gamma'(t)|_gdt= C L_\gamma
    \end{align}
    If $\gamma([a,b])$ is entirely contained in $\overline{V}$, then by definition,
    \begin{align}
        d_h(\varphi(p),\varphi(q))\le L_h(\varphi\circ\gamma)\le CL_\gamma.
    \end{align}
    If $\gamma([a,b])$ is not entirely contained in $\overline{V}$, let
    \begin{align}
        t_{max} = \max\{s:\gamma(t)\subset\overline{V}, \forall t\le s\}<b
    \end{align}
    by definition,
    \begin{equation}
    \begin{aligned}
    d_h(\varphi(p),\varphi(q)) \le \epsilon =& d_h(\varphi(p),\varphi(\gamma(t_{max})))\\
    \le&L_{\varphi\circ\gamma|_{[a,t_{max}]}}\\
    \le& CL_{\gamma|_{[a,t_{max}]}}\\
    \le& CL_{\gamma}.
    \end{aligned}
    \end{equation}
    So no matter whether $\gamma$ is entirely contained in $\overline{V}$ or not, we have $d_h(\varphi(p),\varphi(q))\le CL_{\gamma}$, consequently, 
    \begin{align}
        d_h(\varphi(p),\varphi(q))\le C\inf_{\gamma}L_{\gamma}=Cd_g(p,q).
    \end{align}
    On the other hand, if we choose
    \begin{align}
        \gamma(t) = \varphi^{-1}(\varphi(p)+\frac{t-a}{b-a}(\varphi(q)-\varphi(p)))
    \end{align}
    then $\gamma(a)=p$ and $\gamma(b)=q$, we therefore have
    \begin{align}
        cd_g(p,q)\le cL_\gamma\le L_{\varphi\circ\gamma}=d_h(\varphi(p),\varphi(q)),
    \end{align}
    which completes the proof.
\end{proof}
Now we give the proof of Theorem \ref{compatible}.
\begin{proof}
    First, $(M,d_g)$ is a metric space:
    \begin{itemize}
        \item \textbf{Reflexivity.} Immediately by definition, $d_g(p,p)=0, \forall p\in\M$.
        \item \textbf{Positivity.} $\forall p,q\in\M$, if $p\neq q$, let $(U,\varphi)$ be a chart around $p$ such that $q\notin U$. Choose $\epsilon$ small enough such that $\overline{\varphi^{-1}(B_\epsilon(\varphi(p)))}\subset U$. Then by Lemma \ref{lemma3}, there exists a positive constant $C$ such that $d_g(p,q)\ge C^{-1}d_h(\varphi(p),\varphi(q))>C^{-1}\epsilon>0$.    
        \item \textbf{Symmetry.} Because any curve go from $p$ to $q$ can be reparametrized to go from $q$ to $p$, $d_g(p,q)=d_g(q,p)$.
        \item \textbf{Triangular Inequality.} $\forall p,q,r\in\M$, let $\gamma_1:[a,b]\rightarrow\M$ and $\gamma_2:[a,b]\rightarrow\M$ be any continuous piecewise smooth curve goes from $p$ to $q$ and from $q$ to $r$ respectively. Define $\gamma:[a,b]\rightarrow\M$ as
        \begin{align}
            \gamma(t)=\begin{cases}
                \gamma_1(2t-a), \quad t\in[a,(a+b)/2],\\
                \gamma_2(2t-b)\quad t\in[(a+b)/2,b].\\
            \end{cases}
        \end{align}
        Then $\gamma$ is a continuous and piecewise smooth curve goes from $p$ to $r$, and
        \begin{align}
            d_g(p,r)\le L_\gamma=L_{\gamma_1}+ L_{\gamma_2}.
        \end{align}
        Taking the infimum over all possible $\gamma_1$ and $\gamma_2$ gives
        \begin{align}
            d_g(p,r)\le d_g(p,q)+ d_g(q,r).
        \end{align}
    \end{itemize}
    Second, to show that $(\M,g)$ and $(\M,d_g)$ have the same topology, we only need to show that an open set of either of them is also an open set of the other. 
    \begin{itemize}
        \item Suppose that $U\subset\M$ is open with respect to the manifold topology. For any $p\in U$ and $V=\varphi^{-1}(B_\epsilon(\varphi(p)))$ such that $\overline{V}\subset U$, the positivity argument above implies that
        \begin{align}
            d_g(p,q)\ge C^{-1} \epsilon, \quad \forall q\notin \overline{V}.
        \end{align}
        i.e. $q\in\overline{V}\subset U$ whenever $d_g(p,q)< C^{-1} \epsilon$. Thus, $U$ is open in the metric topology.
        \item On the other hand, suppose $U$ is open in metric topology. For any $p\in U$ and $\epsilon>0$, let $V=\varphi^{-1}(B_\epsilon(\varphi(p)))$, then $V$ is a neighborhood of $p$ in the manifold topology. By Lemma \ref{lemma3}, there exist positive constants $c,C$ such that
    \begin{align}
     cd_g(p,q)\le d_h(\varphi(p),\varphi(q))\le Cd_g(p,q),\quad
         \forall q\in\varphi^{-1}(B_\epsilon(\varphi(p))).
     \end{align}
     So $V$ is a subset of the metric ball $\{q\in\M:d_g(p,q)<c^{-1}\epsilon\}$, since $U$ is open in the metric topology, we can choose $\epsilon$ small enough such that
     $V\subset\{q\in\M:d_g(p,q)<c^{-1}\epsilon\}\subset U$. Therefore, $U$ is open in the manifold topology.
    \end{itemize}
\end{proof}

\subsection{Impossibility of Isometric Embedding of a Hemisphere into $\R^2$}
\label{hemisphere}
\begin{theorem}[Impossibility of Isometric Embedding of a Hemisphere]
Let \(H\subset S^2\) be a hemisphere of the unit sphere in \(\mathbb{R}^3\). Then, there is no isometric embedding. 
\[
F: H \to \mathbb{R}^2.
\]
\end{theorem}

\begin{proof}
Recall that a map \(F: (M,g) \to (N,h)\) between two Riemannian manifolds is called an \emph{isometric embedding} if \(F\) is an embedding (injective immersion) and for every pair of tangent vectors \(u,v\) at a point \(p \in M\), we have
\[
    g_p(u,v) = g^\M_{F(p)}(d_pF(u),d_pF(v))
\]
In particular, this condition ensures that geodesic distances in \(M\) are preserved under \(F\). By \emph{Gauss's Theorema Egregium} \cite{gauss1828disquisitiones}, the Gaussian curvature \(K\) of a 2-dimensional surface is an \emph{intrinsic} invariant \cite{do2016differential}. This means that if two surfaces are isometric, then their Gaussian curvatures must coincide pointwise (under the isometry) \cite{spivak1970comprehensive}.

On the interior of the unit hemisphere \(H\subset S^2\), the Gaussian curvature is strictly positive and, in fact, constant:
\[
K_H \;=\; 1.
\]
On the other hand, any region in the Euclidean plane \(\mathbb{R}^2\) has Gaussian curvature
\[
K_{\mathbb{R}^2} \;=\; 0.
\]
If there existed an isometric embedding \(f: H \to \mathbb{R}^2\), then the induced metric on \(H\) through \(f\) would coincide with the standard metric of the plane on its image. Consequently, by the intrinsic nature of curvature, we would require \(K_H = K_{\mathbb{R}^2}\) at corresponding points. However,
\[
1 \;\neq\; 0.
\]
So, we arrive at a direct contradiction. Hence, no such isometric embedding can exist.
\end{proof}

\begin{remark}
A more intuitive explanation is to note that on surfaces with strictly positive Gaussian curvature, small geodesic triangles have angle sums strictly greater than \(\pi\). In contrast, triangles in the Euclidean plane always have angle sums exactly equal to \(\pi\). An isometry would require these angle sums to match, which is impossible. This discrepancy reflects the deeper fact that Gaussian curvature is an intrinsic property and cannot be altered by bending (or embedding) into a space of different curvature.
\end{remark}

\subsection{Details for Experimental Setup}
\label{details}

\subsubsection{Hyper-Parameters for Baseline Model}
\label{hyperparameter}

To ensure a fair and meaningful comparison, we performed a comprehensive hyperparameter search for all baseline models across the evaluated datasets. These models include GeomAE, GGAE, GRAE, SPAE, and Vanilla AE, each of which relies on different forms of geometric regularization. For every dataset, we report the selected regularization strength, learning rate, and number of training epochs. All baselines were trained with their originally proposed objectives and adapted to our experimental pipeline for consistency. Detailed hyperparameter configurations are listed in Table~\ref{tab:parameter}.

\begin{table}[ht]
\centering
\caption{Settings for hyper-parameters applied to baselines across all evaluated datasets}
\footnotesize 
\begin{tabular}{cccccccc}
\toprule
\textbf{Dataset} & \textbf{Method}& & \boldmath$\lambda_{\text{reg}}$& & \textbf{lr}& & \textbf{n\_epoch} \\
\midrule

 & GeomAE    & & 1e+4& & 1e-3& & 2000 \\
 & GGAE      & & 1e-2& & 1e-3& & 2000 \\
 & GRAE      & & 1e-4& & 1e-3& & 2000 \\
\multirow{-3}{*}{Swiss Roll} & SPAE      & & 10& & 1e-3& & 2000 \\
 & Vanilla AE & & 0& & 1e-3& & 2000 \\
\midrule

 & GeomAE    & & 3& & 1e-3& & 5000 \\
 & GGAE      & & 1e-2& & 1e-3& & 5000 \\
 & GRAE      & & 1e-1& & 1e-3& & 5000 \\
 \multirow{-3}{*}{Toroidal Helix}& SPAE      & & 3& & 1e-3& & 5000 \\
 & Vanilla AE& & 0& & 1e-3& & 5000 \\
\midrule

 & GeomAE    & & 1e-4& & 1e-3& & 500 \\
 & GGAE     &  & 1e-4& & 1e-3& & 500 \\
& GRAE      & & 100& & 1e-3& & 500 \\
\multirow{-3}{*}{Teapot}  & SPAE    &   & 100& & 1e-3& & 500 \\
 & Vanilla AE & & 0& & 1e-3& & 500 \\
\midrule

 & GeomAE   &  & 1e-1& & 1e-3& & 1000 \\
 & GGAE     &  & 10& & 1e-3& & 1000 \\
 & GRAE     &  & 10& & 1e-3& & 1000 \\
\multirow{-3}{*}{dSprites} & SPAE   &    & 1e-2& & 1e-3& & 1000 \\
 & Vanilla AE& & 0& & 1e-3& & 1000 \\
\midrule

 & GeomAE   &  & 1e-4& &  1e-3& & 200 \\
 & GGAE     &  & 1e-2& &  1e-3& & 200 \\
 & GRAE     &  & 1& &  1e-3& & 200 \\
\multirow{-3}{*}{Objective Tracking} & SPAE     &  & 1e-2& &  1e-3& & 200 \\
 & Vanilla AE& & 0& &  1e-3& & 200 \\

\bottomrule
\end{tabular}
\label{tab:parameter}
\end{table}

\subsubsection{Hyper-Parameter for MAE}

For our proposed MAE, we provide detailed hyperparameter settings for both variants: MAE-iso (which enforces local isometric regularization) and MAE-con (which uses local conformal regularization). Both variants share the same reconstruction and global distance losses, but differ in the form of local geometric constraint. We tune the relative weighting coefficients $\lambda_{\text{global}}$, $\lambda_{\text{local}}$, and $\lambda_{\text{diag}}$, as well as the decay rate for $\lambda_{\text{global}}$ in order to balance training stability and geometric fidelity. Learning rates and epoch counts are also tailored to each dataset. The full list of hyperparameters used for each experiment is summarized in Table~\ref{tab:parameter_MAE}.

\begin{table}[ht]
\centering
\caption{Settings for hyperparameters applied to the MAE across all evaluated datasets}
\begin{tabular}{ccccccccccc}
\toprule
\textbf{Dataset} & \textbf{Method} & \textbf{regularization} & \textbf{mode} & \textbf{$\boldsymbol{\lambda}_{\text{local}}$} & \textbf{$\boldsymbol{\lambda}_{\text{global}}$} & \textbf{$\boldsymbol{\lambda}_{\text{diag}}$}  & \textbf{decay} & \textbf{lr} &  \textbf{n\_epoch}\\
\midrule
Swiss Roll & MAE-iso  & Lipschitz loss & isometric & 10 & 100 & 1e-3  & 0.005 & 1e-3 & 2000\\
Swiss Roll & MAE-con  & Lipschitz loss & conformal & 10 & 100 & 1e-3 & 0.005 & 1e-3 & 2000\\
Toroidal Helix & MAE-iso & Lipschitz loss & isometric & 1e-5 & 6 & 1e-5 & 0 & 1e-3 & 5000\\
Toroidal Helix & MAE-con & Lipschitz loss & conformal & 1e-5 & 6 & 1e-5 & 0 & 1e-3 & 5000\\
Teapot & MAE-iso & Lipschitz loss & isometric & 1e-4 & 1 & 1e-3 & 0 & 1e-3 & 500\\
Teapot & MAE-con & Lipschitz loss & conformal & 1e-4 & 1 & 1e-3  & 0 & 1e-3 & 500\\
dSprites & MAE-iso & Lipschitz loss & isometric & 0.1 & 0.9 & 1e-2 & 0 & 1e-3 & 1000\\
dSprites & MAE-con & Lipschitz loss & conformal & 0.1 & 0.9 & 1e-2 & 0 & 1e-3 & 1000\\
Object Tracking & MAE-iso & Lipschitz loss & isometric & 0.1 & 10 & 1e-3 & 0 & 1e-3 & 200\\
Object Tracking & MAE-con & Lipschitz loss & conformal & 0.1 & 10 & 1e-3 & 0 & 1e-3 & 200\\

\bottomrule
\end{tabular}
\label{tab:parameter_MAE}
\end{table}

\subsubsection{Evaluation Metrics}
\label{metrics}
In this subsection, we introduce the two metrics, \emph{kNN} and $KL_\sigma$ ($\sigma \in {0.01,0.1,1}$), used to assess how well the learned latent space preserves the geometry of the original data graph. These metrics compare pairwise distances among original data space to pairwise distances among their corresponding latent representations. For pairwise distances in the original data space, we use Isomap to compute geodesic distances. In synthetic datasets (which already have three dimensions), Isomap is applied directly to that domain. For image-based datasets (dSprite and Object Tracking), each image is first reshaped into a vector before applying Isomap to estimate geodesic distances. Euclidean distances are used for pairwise comparisons in the latent space. The following subsections provide additional details of each metric. We will further illustrate each metric below.
\paragraph{\emph{kNN}}
The kNN metric measures how local neighborhoods are preserved when mapping data to the latent space. This metric computes the fraction of a point’s k-nearest neighbors in the original data that remain its k-nearest neighbors in the latent representation, yielding a maximum score of 1 when all neighbors are perfectly retained. 
\paragraph{\bm{$KL_\sigma$}}
$K L_\sigma$ compares how well densities in the original data space align with densities in the latent space, leveraging the Kullback-Leibler divergence (Chazal et al., 2011). Let $X=\left\{\mathbf{x}_i\right\}_{i=1}^N$ be the original data and $Z=\left\{\mathbf{z}_i\right\}_{i=1}^N$ the corresponding latent points, where distances are measured in Euclidean terms. We define the density estimate on $X$ as
\begin{align}
p_{X, \sigma}\left(\mathbf{x}_i\right)=\frac{\tilde{p}_{X, \sigma}\left(\mathbf{x}_i\right)}{\sum_j \tilde{p}_{X, \sigma}\left(\mathbf{x}_j\right)}
\end{align}
\begin{align}
\tilde{p}_{X, \sigma}\left(\mathbf{x}_i\right)=\sum_j \exp \left(-\frac{1}{\sigma}\left(\frac{\operatorname{dist}\left(\mathbf{x}_i, \mathbf{x}_j\right)}{\max _{\mathbf{x}^{\prime}, \mathbf{x}^{\prime \prime} \in X} \operatorname{dist}\left(\mathbf{x}^{\prime}, \mathbf{x}^{\prime \prime}\right)}\right)^2\right)
\end{align}
An analogous formula defines $p_{Z, \sigma}$ in the same way. The pairwise distance in both original and latent space is used to compute the estimated density. The divergence $K L_\sigma$ is then $D_{\mathrm{KL}}\left(p_{X, \sigma} \| p_{Z, \sigma}\right)$. Larger values of $\sigma$ capture more global structures, while smaller values emphasize local geometry. In practice, we adopt $\sigma \in$ $\{0.01,0.1,1\}$.

\subsection{Details for Experimental Results} \label{app:exp_results}

We now present detailed quantitative comparisons between our proposed MAE variants and a suite of baseline methods across four benchmark datasets: Toroidal Helix, Teapot, Object Tracking, and dSprites. For each dataset, we evaluate performance using five key metrics: reconstruction error (recon.), kNN neighborhood recall, and KL-divergence at three scales (KL$_{0.01}$, KL$_{0.1}$, KL$_{1}$). Lower values in reconstruction and KL-divergence indicate better performance, while higher values in kNN suggest stronger local geometry preservation. In each table, we highlight the best and second-best results for each metric in bold.

\begin{table}[ht]
\centering
\caption{Quantitative comparison of our method against various baselines on the \textbf{Toroidal Helix} dataset. For recon. and KL, lower values indicate better performance, while higher KNN scores are preferable. The best two results are indicated by bold text.}
\begin{tabular}{cccccc}
\toprule
Method & recon. ($\downarrow$) & kNN ($\uparrow$) & KL$_{0.01}$ ($\downarrow$) & KL$_{0.1}$ ($\downarrow$) & KL$_{1}$ ($\downarrow$) \\
\midrule
MAE-iso & \textbf{4.55e-4} & \textbf{8.90e-1} & \textbf{2.03e-3} & 1.35e-2 & 1.62e-3\\
MAE-con & \textbf{3.11e-4} & \textbf{8.90e-1} & \textbf{2.03e-3} & 1.3467e-2 & 1.60e-3\\
\midrule
GeomAE & 1.04e-2 & 7.87e-1 & 2.966e-3 & \textbf{9.31e-4} & \textbf{8.80e-5}\\
GGAE & 1.09e-3 & 8.84e-1 & 3.720e-3 & 6.99e-3 & 4.57e-4\\
GRAE & 1.07e-3 & 7.62e-1 & 2.95e-2 & 1.79e-2 & 2.61e-3\\
SPAE & 1.97e-2 & 8.88e-1 & \textbf{1.97e-3} & 1.27e-2 & 1.43e-3\\
TSNE & NaN & 8.89e-1 & 2.65e-3 & 2.82e-2 & 3.82e-3\\
UMAP & NaN & 3.23e-3 & 7.27e-3 & 2.19e-2 & 3.23e-3\\
Vanilla AE & 1.04e-2 & 7.85e-1 & 4.05e-3 & \textbf{9.58e-4} & \textbf{1.47e-4}\\
\bottomrule
\end{tabular}
\label{tab:numerical_helix}
\end{table}

\begin{table}[ht]
\centering
\caption{Quantitative comparison of our method against various baselines on the \textbf{Teapot} dataset. For recon. and KL, lower values indicate better performance, while higher KNN scores are preferable. The best two results are indicated by bold text.}
\begin{tabular}{cccccc}
\toprule
Method & recon. ($\downarrow$) & kNN ($\uparrow$) & KL$_{0.01}$ ($\downarrow$) & KL$_{0.1}$ ($\downarrow$) & KL$_{1}$ ($\downarrow$) \\
\midrule
MAE-iso & \textbf{3.50e-3} & 7.07e-1 & \textbf{6.81e-3} & \textbf{9.64e-3} & \textbf{9.47e-4}\\
MAE-con & \textbf{3.58e-3} & 7.07e-1 & \textbf{6.81e-3} & 9.65e-3 & \textbf{9.36e-4}\\
\midrule
GeomAE & 5.30e-3 & \textbf{7.60e-1} &1.35 & 8.82e-2 & 1.8e-3\\
GGAE & 4.33e-3 & \textbf{7.73e-1} & 1.36 & 9.24e-2 & 1.92e-3\\
GRAE & 4.88e-3 & 6.97e-1 & 9.61e-3 & 1.39e-2 & 1.21e-3\\
SPAE & 6.11e-3 & 7.07e-1 & 6.94e-3 & \textbf{9.52e-3} & 9.88e-4\\
TSNE & NaN & 3.32e-1 & 1.79e-2 & 2.60e-2& 3.46e-3\\
UMAP & NaN & 3.28e-1 & 1.56e-2 & 1.55e-2 & 1.64e-3\\
Vanilla AE & 5.46e-3 & 7.31e-1 & 8.65e-2 & 5.03e-2 & 2.47e-3\\
\bottomrule
\end{tabular}
\label{tab:numerical_teapot}
\end{table}

\begin{table}[ht]
\centering
\caption{Quantitative comparison of our method against various baselines on the \textbf{Objective Tracking} dataset. For recon. and KL, lower values indicate better performance, while higher KNN scores are preferable. The best two results are indicated by bold text.}
\begin{tabular}{cccccc}
\toprule
Method & recon. ($\downarrow$) & kNN ($\uparrow$) & KL$_{0.01}$ ($\downarrow$) & KL$_{0.1}$ ($\downarrow$) & KL$_{1}$ ($\downarrow$) \\
\midrule
MAE-iso & 9.72e-3 & 8.08e-1 & 2.21e-2 & \textbf{9.18e-2}& 2.37e-3\\
MAE-con & 8.64e-3 & 8.09e-1 & 2.23e-2 & \textbf{9.19e-2} & 2.38e-3\\
\midrule
GeomAE & \textbf{7.60e-3} & 3.17e-1 &1.53e-1 & 1.05e-1 & 2.84e-3\\
GGAE & 7.65e-3 & 3.20e-1 & 1.43e-1 & 1.04e-1 & 2.87e-3\\
GRAE & \textbf{5.43e-3} & 7.80e-1 & 2.33e-1 & 9.42e-2 & 2.58e-3\\
SPAE & 7.80e-3 & \textbf{8.14e-1} & 3.03e-2 & 9.35e-2 & \textbf{2.28e-3}\\
TSNE & NaN & \textbf{8.14e-1} & \textbf{1.07e-2} & 1.22e-1& 2.77e-3\\
UMAP & NaN & \textbf{8.24e-1} & \textbf{1.96e-2} & 1.11e-1 & \textbf{2.26e-3}\\
Vanilla AE & 8.89-3 & 3.67e-1 & 1.10e-1 & 1.58e-1 & 5.02e-3\\
\bottomrule
\end{tabular}
\label{tab:numerical_objective}
\end{table}

\begin{table}[t]
\centering
\caption{Quantitative comparison of our method against various baselines on the \textbf{dSprites} dataset. For recon. and KL, lower values indicate better performance, while higher KNN scores are preferable. The best two results are indicated by bold text.}
\begin{tabular}{cccccc}
\toprule
Method & recon. ($\downarrow$) & kNN ($\uparrow$) & KL$_{0.01}$ ($\downarrow$) & KL$_{0.1}$ ($\downarrow$) & KL$_{1}$ ($\downarrow$) \\
\midrule
MAE-iso & 1.42e-2 & 7.12e-1 & 5.88e-2 & 9.12e-2 & 1.15e-3\\
MAE-con & 2.63e-3& 7.16e-1& 5.53e-2 &9.39e-2 & 1.03e-3\\
\midrule
GeomAE & 6.15e-3 & 7.73e-1 & 1.21 &1.10 & 2.14e-3\\
GGAE & \textbf{2.48e-3} & \textbf{7.87e-1} & 9.35e-2 & \textbf{6.07e-2} & \textbf{5.22e-4}\\
GRAE & 2.59e-3 & 7.81e-1 & 6.17e-2& 7.51e-2 & \textbf{4.71e-4}\\
SPAE & \textbf{2.01e-3} & 7.16e-1 & 4.43e-2 & 1.00 & 1.20e-3\\
TSNE & NaN & 4.52e-1 & \textbf{6.32e-3} & 2.42 & 6.70e-3\\
UMAP & NaN & 4.49e-1 & \textbf{8.23e-3} & 2.40 & 7.14e-3\\
Vanilla AE & 2.57e-3 & \textbf{7.82e-1}& 1.31 & \textbf{6.94e-2} & 1.17e-3\\
\bottomrule
\end{tabular}
\label{tab:numerical_dsprite}
\end{table}


\end{document}